\documentclass{llncs}
\usepackage[latin1]{inputenc}
\usepackage{amsmath}
\usepackage{amsfonts}
\usepackage{amssymb}
\usepackage{tikz}
\usetikzlibrary{automata}
\usepackage{algorithm2e}

\title{Manipulating Tournaments in Cup and Round Robin 
Competitions\thanks{The
second author is funded by the Department of Broadband, Communications
and the Digital Economy and the Australian Research Council. }
}
\author{Tyrel Russell\inst{1} \and Toby Walsh\inst{2}}
\institute{Cheriton School of Computer Science, University of Waterloo, Waterloo, Canada\\ \email{tcrussel@cs.uwaterloo.ca} \and NICTA and UNSW, Sydney, Australia\\ \email{toby.walsh@nicta.com.au}}

\date{}

\begin{document}
\maketitle
\begin{abstract}
In sports competitions, teams can manipulate the result by, 
for instance, throwing games. We show that we can
decide how to manipulate round robin and cup competitions, two 
of the most popular types of sporting competitions 
in polynomial time. In addition, we show that finding the minimal 
number of games that need to be thrown to manipulate the
result can also be determined in polynomial time.    
Finally, we show that 
there are several different variations of standard cup competitions 
where manipulation remains polynomial.
\end{abstract}

\section{Introduction}

The Gibbard-Satterthwaite theorem proves that, under some
modest assumptions, voting systems are always manipulable.
One possible escape proposed by Bartholdi, Tovey and Trick 
is that the manipulation may be computationally
too difficult to find \cite{bartholdi89} (but see \cite{walsh09}
for discussion about whether manipulation is hard 
not just in the worst case).
Like elections, sporting competitions can also be manipulated.
For example a coalition of teams might 
throw games strategically to ensure that a desired team wins or 
a certain team loses. We consider here
the computational complexity of computing such
manipulations. We show that, for several common types of competitions, 
determining when a coalition can manipulate the result
is polynomial.  Our results adapt manipulation procedures
for elections where voters can misrepresent their 
preferences. We consider two of the 
most common methods used for deciding sporting 
competitions, cups and round robins. These
correspond to elections run using sequential majority voting 
(also known as the cup rule) and Copeland scoring, respectively.

Manipulating a sporting competition is slightly
different to manipulating an election as, in a sporting
competition, the voters are also the candidates. 
A tournament graph describes the outcome of all fair games 
between opponents.  Manipulating a competition therefore
modifies not votes but the tournament graph directly. 
Since it is hard without bribery or similar mechanisms
for a team to play better than it can, we consider
manipulations where teams in the coalition are only able
to throw games. By comparison, in an election, voters in
the manipulating coalition can mis-report their
preferences in any way they choose.  Tang, Shoham and Lin \cite{tang09} addressed this type of tournament manipulation in team competitions by providing conditions for truthful reporting of player strengths.  Their method tries to encourage teams to rank their players honestly so that, when the teams compete in bouts, the best player on one team plays the best on the other, the second best plays the opposing second and so forth.  An example of this type of competition is Davis Cup Tennis.

Conitzer, Sandholm and Lang \cite{conitzer07} give an algorithm to
determine if a coalition can manipulate 
the cup rule.  We modify this algorithm to manipulate
directly the tournament graph instead of the votes. 
Bartholdi, Tovey and Trick 
\cite{bartholdi89} discuss direct manipulations of the tournament under second order Copeland, a round robin like rule with secondary tie breaking.  Using the work of Kern and Paulusma \cite{kern04}, we show that the manipulation of round robin competitions is directly tied to the problem of winner determination in sports problems.  
Altman, Procaccia and Tenneholtz \cite{altman09} construct a social choice rule that is monotonic, pairwise non-manipulable and non-imposing.  Round robin and cup competitions are monotonic as a single team losing a game does no better.  Pairwise non-manipulability means that no two teams are better off by manipulating the tournament.  Our results show that round robin and cup competitions are pairwise manipulable and that manipulations can be calculated in polynomial time.

We modify our algorithms to calculate the smallest number of manipulations needed.  For cup competitions, we add dynamic programming to Conitzer, Sandholm and Lang's algorithm. For round robin competitions, we modify the flow network used to solve winner determination to include weights on manipulations and calculate a minimum cost feasible flow.  Vu, Altman and Shoham \cite{vu09} used a similar method to calculate the probability that a team wins the competition.  
 Vu et al.~\cite{vu09} provide several results on determining probabilities of teams winning given a seeding of the tournament.  Hazon et al.~\cite{hazon08} showed that it is NP-Complete to determine if a team wins a cup with a given probability.  This is similar to determining a possible winner given random reseeding except edges in the tournament are labelled with probabilities.  We look at the complexity of manipulation under reseeding in the deterministic case.  Finally, we look at the complexity of double elimination cups. 


\section{Background}


In many sporting competitions, the final winner of a competition is decided by a tree-like structure, called a \emph{cup}.  The most common type is a \emph{single elimination cup}, a tree structure where the root and internal nodes represent games and leaves represent the teams in the tournament.  
A cup can include a \emph{bye game}, a game where a team skips a game to re-balance the schedule.  Usually, the top teams are given a bye game while the lower teams do not so that the number of teams in the next round is strict power of two.
Cups need to be \emph{seeded} to determine which teams play against each other in each round.  One method for seeding is by rank.  The most common method for ranked seeding or reseeding is to have the top team play the worst team, the second place team play the second worst team and so forth.  An example of ranked seeding using this method is the National Basketball Association in the US.  Another method for determining seeding is randomly, also known as a draw.  An example of this is the UEFA Champions League where teams reaching the quarter finals are randomly paired for the remainder of the tournament.  Seeding may also be more complex (for instance, it may
be based on the group from which teams qualify or some other criteria).  
Another way that cups are modified is between fixed and unfixed cups.  A \emph{fixed cup} is a cup where there is a single seeding at the start of the cup.  Examples of this are the National Basketball Association and the World Cup of Football.  An \emph{unfixed cup} is one where seeding may occur not only before the start but between any round.  Examples with an unfixed cup are the National Hockey League and the UEFA Champions League.

Cups are not necessarily single elimination.  A double elimination cup is designed so that a team can lose two games instead of one game.  If a team loses, they play other teams that have also lost until they lose a second time or they win the final game of the tournament.  These tournaments are organized as two cups where losers enter the second cup at various stages depending on when they lose their first game.
%
Finally, a \emph{round robin competition} is a competition where each team plays every other team a given number of times.  In a \emph{single round robin competition}, each team plays every other team exactly once.  Another common variant of this is for teams to play a \emph{double round robin competition} where each team plays every other team twice, often at home and away. 

\section{Manipulating the Tournament}

A \emph{tournament} is a directed graph $G = (V,E)$ where the underlying undirected graph is a complete graph.  We assume that the tournament is available for the remainder of the paper.  Every directed edge $(v_i,v_j) \in E$ represents a victory by $v_i$ over $v_j$.  The number of the teams in the competition is $\left|V\right|=m$.
We define a \emph{manipulation} of the tournament as any replacement of an edge $(v_i,v_j)$ in the graph with the edge $(v_j,v_i)$.  This is equivalent to a manipulation of votes but here we are changing the winner directly instead of just changing the vote.  Note that, as in election manipulation where the electoral vote is assumed to be known, we assume that we know, via an oracle, the relative strengths of teams and can represent the winner of the contests in the tournament graph.
We restrict manipulations by only allowing the manipulation of an edge $(v_i,v_j)$ if candidate $v_i$ is a member of the coalition.  This restricts the behaviour of the manipulators to throwing games where they could have won.  This restriction is due to the fact that it is simple to perform worse but more difficult to play better.
We consider two different types of manipulations.  A \emph{constructive manipulation} is one that ensures a specific team wins the competition.  A \emph{destructive manipulation} is one that ensures a specific team loses the competition.
For round robin competitions, we generalize the concept of the tournament beyond the simple win-loss scoring model to a complete graph where the edge $(v_i,v_j)$ has a non-negative weight $w_{ij}$ which represents the number of points that would be earned by $v_i$ when playing $v_j$ in a fair game.  We define a manipulation in this case as an outcome where the points earned in the match are different to those given by the tournament.  However, manipulations are restricted so that the manipulator achieves no more points and the team being manipulated achieves no less points.

In this section, we restrict ourselves to fixed cups with a known seeding.  We also look just at single round robin tournaments though the results generalize to multi-round robin tournaments. 

\subsection{Cup Competitions}

For cup competitions,
finding a constructive or destructive
manipulation of the tournament is polynomial.  Our results make
use of results in \cite{conitzer07} which shows that a manipulation 
of an election using the cup rule can be found in $O(m^{3}n)$ time where $m$ is the number of candidates and $n$ is the number of voters.

\begin{theorem}
\label{cup-constructive}
Determining if a cup competition can be constructively manipulated
using manipulations of the tournament takes polynomial time.
\end{theorem}

\begin{proof}
This proof is a bottom up version the proof of Theorem 2 from Conitzer, Sandholm and Lang (CSL)\cite{conitzer07} but substitutes tournament manipulations for voting manipulations.  The basic CSL algorithm is a recursive method that treats each node in the tree (which is not a leaf) as a sub-election (see Algorithm {CSL}).  Conitzer et al.~\cite{conitzer07} note that a team wins a sub-election if and only if they must win one of its children and they can defeat one of the potential winners on the other side.  It is perhaps simpler to understand this algorithm from a bottom up perspective.  Observe that if we have two leaf nodes $v_i$ and $v_j$ and there exists an arc in the tournament $(v_i,v_j)$ then $v_i$ wins the match and is a potential winner of the sub election between $v_i$ and $v_j$.  Now suppose that $v_i$ is in the member of the coalition so it is possible for them to replace $(v_i,v_j)$ with $(v_j,v_i)$ in the tournament and therefore $v_j$ is also a potential winner of the sub-election via manipulation.  Assume we have some sub-election in the middle of the tournament with two sets of potential winners $A$ and $B$.  Any team from $A$ is a potential winner of the sub-election if there exists a team in $B$ that they can defeat or if a coalition member in $B$ throws a game.  The same is true for teams in $B$.  Therefore, there is a constructive manipulation if the desired winner is a member of the potential winners at the top node in the cup tree. 

The original algorithm looked at $O(m^{2})$ pairs of opponents as no two teams were compared more than once.  Note that the original analysis provided a looser $O(m^3)$ bound on the number of comparisons, but this can be tightened by an observation of Vu et al.~\cite{vu09}.  The difference between direct manipulation of the tournament and the method by Conitzer, Sandholm and Lang is that determining if a team could defeat another team meant summing all values of the $n$ voters requiring $O(n)$ time whilst in the direct manipulation of the tournament this can be done in constant time.  Therefore, constructive manipulation of the tournament under the cup rule takes just $O(m^{2})$ time.\qed
\end{proof}

\begin{algorithm}[t]
\Titleofalgo{CSL($v_w$,$c$,$T$,$C$)}
\label{ElimSetFlowAlgo}
\SetArgSty{textrm}
\SetKwData{Winners}{winners}
\SetKwData{Graph}{G}
\SetKwFunction{PW}{PossibleWinners}
\SetKwInOut{Input}{input}
\SetKwInOut{Output}{output}

\Input{A team $v_w$, a cup tree $c$, a tournament graph $T$, and a coalition of teams $C$}
\Output{Returns true if $v_w$ can win via manipulation and false otherwise}
\BlankLine
\SetVline
\Winners $\leftarrow$ \PW{$c$,$T$,$C$}\;
\eIf{$v_w \in \Winners$}
{
	\Return{true}\;
}
{
	\Return{false}\;
}
\end{algorithm}

\begin{procedure}[t]
\Titleofalgo{PossibleWinners($c$,$T$,$C$)}
\SetKwData{Winners}{winners}
\SetKwData{lw}{LeftWinners}
\SetKwData{rw}{RightWinners}
\SetKwFunction{PW}{PossibleWinners}
\SetKwFunction{Add}{add}
\SetKwFunction{Right}{right}
\SetKwFunction{Left}{left}
\SetKwFunction{Leaf}{leaf}
\SetKwInOut{Input}{input}
\SetKwInOut{Output}{output}

\Input{A cup tree $c$, a tournament graph $T$ and a coalition of teams $C$}
\Output{Returns the set of possible winners of the cup tree via manipulation of the tournament by the coalition}
\BlankLine
\SetVline

\eIf{ \Leaf{$c$} }{
	\Return{$\{ c \}$}\;
}{
	\Winners $\leftarrow \{ \}$\;
	\lw $\leftarrow$ \PW{\Left{$c$} ,$T$,$C$}\;
	\rw $\leftarrow$ \PW{\Right{$c$} ,$T$,$C$}\;
	\ForAll{$v_i \in \lw$}
	{	
		\If{$\exists v_j \in \rw \quad \mbox{such that} \quad (v_i,v_j) \in E \lor v_j \in C$}
			{
				\Add{\Winners,$v_i$}\;
			}
	}
	\ForAll{$v_j \in \rw$}
	{	
		\If{$\exists v_i \in \lw \quad \mbox{such that} \quad (v_j,v_i) \in E \lor v_i \in C$}
			{
				\Add{\Winners,$v_j$}\;
			}
	}
	\Return{\Winners}\;
}
\end{procedure}

We observe that destructive manipulation of a competition using tournament manipulations is similar since this simply requires determining if there is at least one other possible winner of the tournament via manipulations.

\begin{theorem}
Determining if a cup tournament can be destructively
manipulated using tournament manipulations takes polynomial time.
\end{theorem}

\begin{proof}
We just determine if we can constructively manipulate the tournament for each other team in turn than the one we wish to lose.\qed
\end{proof}

\subsection{Round Robin Competition}

For round robin competitions, manipulations of the tournament
can be computed in polynomial time for a restricted class of scoring models.  We define a \emph{scoring model} to be the set of tuples giving the possible outcomes of a game.  Copeland scoring has a simple win-loss ($\{(0,1),(1,0)\}$) scoring model where the wining team earns one point and the losing team earns none.  Bartholdi, Tovey and Trick\cite{bartholdi89} showed that constructive manipulation can be determined in polynomial time for a chess scoring model ($\{(0,1),(\frac{1}{2},\frac{1}{2}),(1,0)\}$).  Faliszewski et al.~\cite{faliszewski08} showed that for a range of scoring models manipulating Copeland voting is NP-Complete.

First, we discuss the problem of determining which games need to be manipulated to ensure that a given team $v_{w}$ wins the competition.  Clearly, there are some games that cannot be affected by the coalition and are fixed.  All other games are manipulable.  Games between coalition members can earn any of the possible scores allowed by the scoring model.  We restrict games against non-coalition members by only allowing the manipulator to earn less points and the non-member earns more.  Determining if a given team can be made a winner is analogous to determining if a team wins a round robin tournament when the fixed games have been played and the manipulable games have not been played.  The restriction of the outcomes on games between coalition and non-coalition members requires that the games have outcomes within only a subset of the scoring model.  Using this observation, we obtain the following theorem.

\begin{theorem}
\label{round-robin-constructive}
Determining if there exists a constructive manipulation of a round robin competition is polynomial if the normalized scoring model is of the form $S=\{(i,n-i) \mid 0 \leq i \leq n\}$ and NP-complete, otherwise.
\end{theorem}

\begin{proof}
This proof uses the equivalence of determining whether a team can win a tournament and determining if a constructive manipulation exists with a set of fixed and manipulable games.  Note that a game between a non-coalition member $v_i$ and a coalition member $v_j$ is unfixed but the scores that can be assigned are restricted.  When the scoring model is of the form $S=\{(i,n-i) \mid 0 \leq i \leq n\}$ and the initial result of the game is $(c_i,c_j)$, then the remaining valid scores that can be assigned are those from $(c_i,c_j)$ to $(n,0)$.  By normalizing this new model, we obtain one in which the non-coalition member earns $c_i$ points by default and the result of the game is scored from the model $\{(0,c_j),\ldots,(n-c_i,0)\}$ which is of the form $S=\{(i,n-i) \mid 0 \leq i \leq n\}$.  Kern and Paulusma \cite{kern04} showed that determining if a team can win a tournament (i.e. is not eliminated from competition) takes polynomial time if the normalized scoring model is of the form $S=\{(i,n-i) \mid 0 \leq i \leq n\}$ and is NP-complete otherwise.\qed
\end{proof}


By comparison, it is always polynomial to determine if a destructive 
manipulation exists. 

\begin{theorem}
Determining if there is a destructive manipulation of a round
robin competition takes polynomial time.
\end{theorem}

\begin{proof}
Assume that $v_l$ is the team that the coalition desires to lose.  It is sufficient to check whether the maximum points of another team via manipulation is greater than the points of $v_l$.  If $v_l$ is a member of the coalition and therefore a manipulator, for each team $i$ that we check for points, we apply only manipulations that increase the relative points between $i$ and $v_l$.  For all other teams, we apply the manipulation which decreases the points of $v_l$ the most.  If $v_l$ is not a member of the coalition, no games involving $v_l$ may be manipulated since we restrict manipulations to allow only those manipulations that increase the points of $v_l$ and increase the relative gap between $v_l$ and the manipulator.  Therefore, no other team is better off when games involving $v_l$ are manipulated.  In both cases, we apply the manipulation that increase the points of the team under consideration against all other teams.  If the total number of points of any other team is greater than the points of $v_l$ under these manipulations, then there is a destructive manipulation of $v_l$.  This algorithm can be run in $O(n^2)$ time.\qed
\end{proof}

A further complication is when the goal of manipulation is just to earn a berth in the next round of the playoffs.  It is NP-hard to decide these questions under most playoff systems for all scoring models \cite{mccormick99,gusfield02}.

\section{Minimizing Manipulations}

The number of manipulations required is an important factor.  
It may be advantageous for the coalition to manipulate as few games as possible to avoid detection or to minimize the cost of bribing players.  We show that there is a polynomial algorithm to calculate manipulations which throw a minimal number of games.  This highlights the vulnerability of the two most common types of competitions in sports to manipulation.

\subsection{Minimal Number of Manipulations for Cup Competitions}

Computing the minimal number of manipulations
simply requires keeping a count within our algorithm for 
computing a manipulation. 
We give some notation to identify a specific sub-election in the cup.  We let $s_{\ell}^{v_i}$ be the sub-election at level $\ell$ where $v_i$ is a leaf node of a sub tree below $s_{\ell}^{v_i}$.  We denote the level as the height from the bottom of the cup tree, which is assumed to be a perfect binary tree.  We also define level 0 to be the level belonging to the leaves.
We have $m^{2}$ constants $c_{ij}$ that are 1 if $(v_j,v_i) \in M$ and 0 otherwise, where $M \subseteq E$ is the set of edges which can be manipulated by the coalition.  This corresponds to $c_{ij}=1$ when a manipulation must occur for $v_i$ to win and 0 otherwise.
Finally, we define the minimal number of manipulations needed to win a sub-election $s_{\ell}^{v_i}$, $m(v_i,s_{\ell}^{v_i})$, to be sum of the minimal number of the manipulations for $v_i$ to win one of the children of $s_{\ell}^{v_i}$, and the minimum number of manipulations plus $c_{ij}$ over all possible winners of the other child which $v_i$ can defeat.  We denote the set of teams that $v_i$ can defeat either as described in the tournament or by manipulation as $D_i$.  More formally, the minimal number of manipulations for $v_i$ at $s_{\ell}^{v_i}$ ($\ell \geq 0$) is given by:

\begin{equation}
m(v_i,s_{\ell}^{v_i}) = \left \{ \begin{array}{ll}
0 & \quad \textrm{if}\,\ell = 0\\
m(v_i,s_{\ell-1}^{v_i}) + \min_{v_j\in D_i}(m(v_j,s_{\ell-1}^{v_j})+c_{ij})& \quad \textrm{if}\,\ell > 0
\end{array} \right. .\nonumber
\end{equation}

\begin{lemma}
\label{lemma-cup-min}
The minimal number of manipulations needed to make a team $v_i$ a winner at level $n$ in the tree is equal to $m(v_i,s_{n}^{v_i})$.
\end{lemma}

\begin{proof}
By induction.  First, observe that the minimal number of manipulations at a leaf is 0.  Hence, $m(v_i,s_{0}^{v_i})=0$ for all leaves $v_i$.  Next note that at level 1 there are only 2 nodes in the possible winner sets of the leaves.  Therefore if $v_i$ can defeat $v_j$, $m(v_i,s_{1}^{v_i}) = m(v_i,s_{0}^{v_i}) + m(v_j,s_{0}^{v_j}) + c_{ij} = c_{ij}$ which is the exact number of manipulations that have occurred to make $v_i$ a possible winner so far.  We assume the premise for $1 < n \leq k$.  Now, $m(v_i,s_{k+1}^{v_i}) = m(v_i,s_{k}^{v_i}) + \min_{v_j\in D_i}(m(v_j,s_{k}^{v_j})+c_{ij})$. We know that $m(v_i,s_{k}^{v_i})$ is the minimal number of manipulations for $v_i$ up to level $k$ by the assumption and, for every $v_j\in D_i$, we know that $m(v_j,s_{k}^{v_j})$ is also the minimal number of manipulations for each $v_j$ up to level $k$.  By definition, $c_{ij}$ is the number of manipulations for $v_i$ to defeat $v_j$.  Since $v_i$ can defeat any $v_j$ in $D_i$, the one with the fewest previous manipulations to reach $k$ plus $c_{ij}$ leads to the fewest manipulations in total to make $v_i$ win the sub election $s_{k+1}^{v_i}$.  This equals the minimum over the set $D_i$. Therefore the lemma holds for $k+1$ and, by induction, all $n$ levels of the tree.\qed
\end{proof}

\begin{theorem}
\label{theorem-cup-min}
A modified CSL algorithm, where the team which minimizes the value of $m(v_i,s_{n}^{v_i})$ is selected to lose to team $v_i$ at every node $s_{n}^{v_i}$, calculates the minimal number of manipulations needed to constructively or destructively manipulate a cup competition in polynomial time.
\end{theorem}

\begin{proof}
By Lemma \ref{lemma-cup-min}, the value of $m(v_w,s_{n}^{v_w})$ at the root node is the minimal number of manipulations which ensures $v_w$ is the winner.  Hence, we just need to show that the algorithm remains polynomial.  The modified CSL algorithm still makes $O(m^2)$ comparisons.  The only difference is that we have to calculate the minimum which can be done by storing the minimum as each team is checked.  Therefore, the time complexity remains $O(m^2)$ and calculating the minimum is polynomial.  Constructive manipulation requires calculating $m(v_w,s_{n}^{v_w})$ whilst destructive manipulation requires the minimum over all other teams.\qed
\end{proof}

\subsection{Minimal Number of Manipulations for Round Robin Competitions}

We consider here just Copeland scoring. We conjecture that similar methods 
could be developed for other scoring schemes. 

\begin{definition}
Given a tournament $T=(V,E)$ where $V = \{v_1,\ldots,v_n\}$, a set of manipulable edges $M \subseteq E$,and a distinguished node $v_{w}$, the Minimal Number of Manipulations under Copeland Scoring is the problem of determining the minimal number of edges in $M$ that can be reversed such that $\forall_{v_k\in V,v_w \neq v_k}\\ outDegree(v_{w}) \geq outDegree(v_{k})$.
\end{definition}

Note that Copeland Scoring is the simple win-loss method of scoring where the winning team earns 1 point and the losing team earns 0 points.  Before we show how to calculate the minimal number of manipulations, we show that we can determine the out degree, i.e.~the Copeland Score, of the distinguished node using a minimal number of manipulations in isolation with a greedy algorithm.  The intuition behind this is that we select manipulations to increase the out degree of $v_w$. 

\begin{lemma}
The value of $outDegree(v_{w})$ can be determined in isolation by greedily using, in sequence, a minimal number of manipulations of edges $(v_i,v_w) \in M$ where $\forall_{(v_j,v_w)\in M, v_j \neq v_i} outDegree(v_i) \geq outDegree(v_j)$ until $\forall_{v_k\in V,v_w\neq v_k}\\ outDegree(v_{w}) \geq outDegree(v_{k})$.
\end{lemma}

\begin{proof}
First, we prove that it always uses the least number of manipulations to increase the out degree of $v_w$. To reduce the out degree of two or more nodes that have an out degree larger than $v_w$, it takes at least two manipulations but to increase the out degree of $v_w$ by the same amount takes just one.  For a single node, it is preferred to use the manipulation involving $v_w$ since the other node may increase the out degree of another node requiring more manipulations.  Therefore, using manipulations involving $v_w$ is most efficient. 

Now we show that we never overshoot the stopping criteria and use more than a minimal number of manipulations.  Assume that we use more than the minimal number of manipulations.  This means that we selected an edge that did not decrease the maximum out degree when there existed an edge that would have decreased the maximum out degree of all nodes that we did not select.  However, since we always selected the edge where the source node had the maximal out degree within $M$, we always decreased the maximum out degree whenever possible.  This is a contradiction and the greedy algorithm only uses a minimal number of manipulations when reaching the stopping condition.\qed
\end{proof}

\begin{theorem}
Determining the minimal number of tournament manipulations required under Copeland Scoring takes polynomial time.
\end{theorem}

\begin{proof}
We define $c$ to be the out degree of the distinguished node, $v_w$, calculated using the greedy algorithm.  This corresponds to the number of wins earned by $v_w$.  If the stopping condition has not been reached, we must use $c$ to determine how many more manipulations are necessary.  We construct a winner determination flow graph as described by Kern and Paulusma \cite{kern04} and Gusfield and Martel\cite{gusfield02}(See Fig.~\ref{tournament-flowgraph}, for example).  We add a weight of 1 to each edge $(v_i,v_j)$ where $(v_i,v_j)\notin V$ and therefore represents a manipulation.  All other edges have the weight 0.  The feasible flow which uses the fewest of the non-zero edges is the minimal number of tournament manipulations to achieve a constructive manipulation.  Since the value of $c$ can be determined in a linear number of steps, we only need to do a single min cost flow computation, which is polynomial, to determine the remainder of the minimum number of manipulations necessary to make $v_w$ the team with the highest Copeland score.\qed
\end{proof}

\begin{example}
An example tournament can be seen in Fig.~\ref{tournament-flowgraph}a.  There are 5 teams in this tournament: $v_0$ to $v_4$.  Suppose teams $v_0$ and $v_3$ form a coalition to manipulate the tournament so that $v_0$ wins.  We want to determine the minimum number of manipulations needed to ensure that $v_0$ is the winner.  This requires switching any of the arcs where team $v_3$ wins.  We know that the value of $c=2$ since none of $v_0$'s edges are manipulable in $v_{0}$'s favour.  We construct the graph seen in Fig.~\ref{tournament-flowgraph}b to determine for $c=2$ if there is a feasible solution.  The solution returned has a minimum cost which is equal to the minimum number of manipulations needed to get a feasible flow with the value $c$ plus any used in the greedy algorithm.
\end{example}

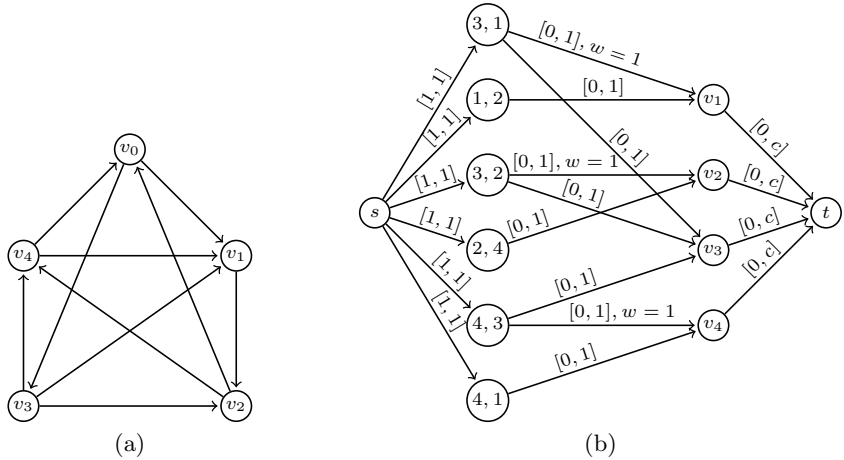
\begin{figure}[t]
\begin{tabular}{p{0.05\textwidth}c p{0.1\textwidth}c}
&
\begin{tikzpicture}[->,shorten >=1pt,%
auto,node distance=2cm,semithick,
inner sep=1pt,bend angle=15]
\tikzstyle{every node}=[font=\scriptsize]
\node[state,minimum size=4mm] (V0) [] {$v_{0}$};
\node[state,minimum size=4mm] (V1) [below right of=V0] {$v_{1}$};
\node[state,minimum size=4mm] (V2) [below of=V1] {$v_{2}$};
\node[state,minimum size=4mm] (V4) [below left of=V0] {$v_{4}$};
\node[state,minimum size=4mm] (V3) [below of=V4] {$v_{3}$};

\tikzstyle{every node}=[font=\scriptsize]
\path(V0) edge node {} (V1) 
	(V0) edge node {} (V3)
	(V1) edge node {} (V2)
	(V2) edge node {} (V0)
	(V2) edge node {} (V4)
	(V3) edge node {} (V1)
	(V3) edge node {} (V2)
	(V3) edge node {} (V4)
	(V4) edge node {} (V0)
	(V4) edge node {} (V1);
\end{tikzpicture}
& &
\begin{tikzpicture}[->,shorten >=1pt,%
auto,node distance=1cm,semithick,
inner sep=1pt,bend angle=30]
\tikzstyle{every node}=[font=\scriptsize]
\node[state,minimum size=4mm] (B1) {$3,1$};
\node[state,minimum size=4mm] (B2) [below of=B1] {$1,2$};
\node[state,minimum size=4mm] (B3) [below of=B2] {$3,2$};
\node[state,minimum size=4mm] (B4) [below of=B3] {$2,4$};
\node[state,minimum size=4mm] (B5) [below of=B4] {$4,3$};
\node[state,minimum size=4mm] (B6) [below of=B5] {$4,1$};
\node[state,minimum size=4mm] (A)  [below left of=B2, node distance=2.125cm]{$s$};
\node[state,minimum size=4mm] (C1) [right of=B2,node distance=3cm] {$v_1$};
\node[state,minimum size=4mm] (C2) [right of=B3,node distance=3cm] {$v_2$};
\node[state,minimum size=4mm] (C3) [right of=B4,node distance=3cm] {$v_3$};
\node[state,minimum size=4mm] (C4) [right of=B5,node distance=3cm] {$v_4$};
\node[state,minimum size=4mm] (D) [right of=A,node distance=6cm] {$t$};
\tikzstyle{every node}=[font=\scriptsize]
\path (A) edge node [sloped,pos=0.8] {$[1,1]$} (B1) 
	(A) edge node [sloped,pos=0.95] {$[1,1]$} (B2)
	(A) edge node [sloped,pos=0.95] {$[1,1]$} (B3)
	(A) edge node [sloped,pos=0.3] {$[1,1]$} (B4)
	(A) edge node [sloped,pos=0.45] {$[1,1]$} (B5)
	(A) edge node [sloped,pos=0.45] {$[1,1]$} (B6)
	(B1) edge node [sloped,midway] {$[0,1]$} (C3)
	(B1) edge node [sloped,very near start] {$[0,1],w=1$} (C1)
	(B2) edge node [sloped,midway] {$[0,1]$} (C1)
	(B3) edge node [sloped,pos=0.3] {$[0,1],w=1$} (C2)
	(B3) edge node [sloped,near start] {$[0,1]$} (C3)
	(B4) edge node [sloped,near start] {$[0,1]$} (C2)
	(B5) edge node [sloped,pos=0.6] {$[0,1],w=1$} (C4)
	(B5) edge node [sloped,midway] {$[0,1]$} (C3)
	(B6) edge node [sloped,midway] {$[0,1]$} (C4)
	(C1) edge node [sloped,pos=0.2] {$[0,c]$} (D)
	(C2) edge node [sloped,pos=0.1] {$[0,c]$} (D)
	(C3) edge node [sloped,pos=0.7] {$[0,c]$} (D)
	(C4) edge node [sloped,pos=0.7] {$[0,c]$} (D);
\end{tikzpicture}
\\
&(a) & & (b)
\end{tabular}
\caption{(a) The tournament graph for five teams.  The distinguished node in the example is $v_0$ which has formed a coalition with $v_3$.  The manipulable edges are $(v_3,v_1)$, $(v_3,v_2)$, $(v_3,v_4)$, $(v_0,v_1)$ and $(v_0,v_3)$.  Edges $(v_1,v_2)$ $(v_2,v_4)$ and $(v_4,v_1)$ cannot be manipulated by the coalition. (b) The min cost flow graph used to calculate the minimum number of manipulations for a given value constructed from the tournament in Fig.~\ref{tournament-flowgraph}a.  The distinguished team is $v_0$, $c=2$ and all weights not shown are 0.}
\label{tournament-flowgraph}
\end{figure}

\section{Reseeding}




If we add multiple seeding rounds then computing a manipulation appears 
difficult.  Recall that ranked reseeding matches the best remaining teams against the worst remaining teams in each round.  The CSL algorithm cannot therefore be applied and a general solution is not known.  However, if the size of the coalition is a constant $c$, then we can determine a manipulation in polynomial time.

\begin{theorem}
\label{ranked-reseeding-cup}
For a ranked reseeding cup competition, 
if the manipulating coalition is of bounded size $c$, 
then determining a set of manipulations that 
makes a team win takes polynomial time.
\end{theorem}

\begin{proof}
The key observation is that with a constant sized coalition there are only a polynomial number of ways to manipulate the games by rearranging the tournament graph.  It suffices to check the winner of each of the polynomial number of fixed tournament graphs.  For each fixed tournament graph, the winner can be determined in linear time as there are only $O(m)$ matches to check.

We show that there are only a polynomial number of different arrangements of manipulations. First note that at most $c$ of the $\frac{m}{2}$ matches in the first round have more than one team as a possible winner.  This means that there is at most $2^{c}$ possibilities to examine after each round.  As there are $\log(m)$ rounds,
we consider at most $(2^{c})^{\log{m}}$ (=$m^c$) possibilities. 
Hence there are at most $O(m^c)$ arrangements of manipulations for an unfixed cup with ranked reseeding and a constant sized coalition. It is sufficient to check each arrangement, which can be done in linear time.
This gives a polynomial algorithm for bounded $c$.\qed
\end{proof}

With random reseeding the problem can be separated into two issues: determining whether manipulation is possible to make a team a winner under every possible seeding and determining if there exists any seeding such that the coalition can manipulate the games to make a given team the winner.  It is unknown whether either of these problems have polynomial algorithms.  Vu et al.~\cite{vu09} and Hazon et al.~\cite{hazon08} tackle some probabilistic variants of possible winners without manipulation of games.  However, the complexity of determining possible winners with a win-loss tournament graph in balanced cup trees remains open \cite{lang07,hazon08,pini08}. 

\section{Double Elimination Competitions}

In a double elimination competitions, a manipulation of the tournament does not automatically bounce the manipulator out of the tournament as in the single elimination case.  However, it does guarantee that the manipulator will be bounced to the secondary bracket from the primary bracket on the first manipulation and out of the tournament on the second manipulation.  As in the case of ranked reseeding, a general solution is not known but there is a polynomial algorithm for double elimination tournaments if the coalition is of constant size.

\begin{theorem}
For double elimination tournaments, if the coalition is of a constant size $c$, determining whether there is a constructive manipulation takes polynomial time.
\end{theorem}

\begin{proof}
This proof follows similar lines as the proof for ranked reseeding.  We will show that there is a polynomial number of manipulation scenarios which can be checked in linear time.
If there is a coalition of size $c$ then a team can manipulate the cup only once if they wish to win the tournament and twice if they desire another team to win.  At each step in the tree, a team must decided whether they wish to manipulate or not.  Before and after they have manipulated once, there remains $c$ teams which can manipulate.  Only after they have manipulated a second time are they removed from the competition.  This means there are at most $2^{c}$ manipulations at each of the $log{m}$ levels.  This gives us $O(2^{\log{m}c}) (=O(m^{c}))$ possibilities that can be checked in linear time, which gives a polynomial algorithm for determining if there is a constructive manipulation.\qed
\end{proof}

\section{Conclusions and Open Problems}

In sporting tournaments, teams can directly manipulate the
tournament graph. We showed that algorithms used to compute
manipulations of votes in elections can be modified to determine the manipulations needed of the tournament graph. 
We proved that such direct manipulation of the fixed cup and round robin competitions can be computed in polynomial time.  In a similar way, we can determine the minimal number of manipulations needed.  For ranked reseeding of cup
competitions, we showed that it is easy to calculate the number of manipulations if the size of the manipulating coalition is bounded by a constant.  We
also gave a polynomial time algorithm for double elimination tournaments for a constant sized coalition.
A number of open question remain.  The manipulation of various variations of the cup competition have unknown complexity including the ranked and random cup competitions.  For random cup competitions, the complexity of manipulation is also unknown if the size of the coalition is bounded.  Similarly, the complexity of manipulating double elimination competitions is still undetermined when the size of the coalition is unbounded.


\bibliographystyle{plain}

\end{document}